\documentclass[twoside]{article}
\usepackage{amsmath,amssymb,stmaryrd}
\usepackage{amsthm}
\usepackage{graphicx}
\usepackage{enumitem}
\usepackage{float}

\def\fro{{\scriptscriptstyle \mathrm{F}}}

\def\xb{{\mathbf x}}

\def\ub{{\mathbf u}}

\def\etab{{\boldsymbol\eta}}

\def\oneb{{\mathbf 1}}
\def\zerob{{\mathbf 0}}

\def\Xb{{\mathbf X}}

\def\Yb{{\mathbf Y}}

\def\Real{{\mathbb{R}}}

\newcommand{\SET}[1]{\{1,\cdots,#1\}}

\def\Ocal{\mathcal{O}}
\def\Rcal{\mathcal{R}}
\def\Bcal{\mathcal{B}}

\def\Ccal{\mathcal{C}}

\def\Xcal{\mathcal{X}}
\def\Lcal{\mathcal{L}}
\def\Diag{{\mathrm{Diag}}}

\def\defin{\triangleq}

\newtheorem{theorem}{Theorem}
\newtheorem{lemma}{Lemma}

\newtheorem{corollary}{Corollary}

\usepackage[accepted]{myStyle}
\usepackage{url}

\begin{document}

\runningtitle{}

\runningauthor{R.Jenatton, J.Huang, C. Archambeau}

\twocolumn[

\aistatstitle{Adaptive Algorithms for Online Convex Optimization \\with Long-term Constraints}

\aistatsauthor{Rodolphe Jenatton$^{*}$ \And Jim Huang$^{\dag}$ \And C\'edric Archambeau$^{*}$}

\aistatsaddress{Amazon, Berlin$^{*}$  \And Amazon, Seattle$^{\dag}$ \And \texttt{\{jenatton,huangjim,cedrica\}@amazon.com}} ]

\begin{abstract}
We present an adaptive online gradient descent algorithm to solve online convex optimization problems with long-term constraints, which are constraints that need to be satisfied when accumulated over a finite number of rounds $T$, but can be violated in intermediate rounds. For some user-defined trade-off parameter $\beta \in (0,1)$, the proposed algorithm achieves cumulative regret bounds of $\Ocal(T^{\max\{\beta,1-\beta\}})$ and $\Ocal(T^{1-\beta/2})$ for the loss and the constraint violations respectively. Our results hold for convex losses and can handle arbitrary convex constraints without requiring knowledge of the number of rounds in advance. Our contributions improve over the best known cumulative regret bounds by Mahdavi,~et~al.~(2012) that are respectively ${\cal O}(T^{1/2})$ and ${\cal O}(T^{3/4})$ for general convex domains, and respectively ${\cal O}(T^{2/3})$ and ${\cal O}(T^{2/3})$ when further restricting to polyhedral domains. We supplement the analysis with experiments validating the performance of our algorithm in practice.\vspace{-0.2in}

\end{abstract}
\section{Introduction\vspace{-0.1in}}
\indent Online convex optimization (OCO) plays a key role in machine learning applications, such as adaptive routing in networks \cite{Awerbuch2008} and online display advertising \cite{Agrawal2015}. In general, an OCO problem can be formulated as a sequential and repeated game between a learner and an adversary. In each round $t$, the learner first plays a vector $\mathbf{x}_t \in {\cal X}  \subseteq \mathbb{R}^p$, where ${\cal X}$ is a compact convex set corresponding to the set of possible solutions. The learner then incurs a loss $f_t(\mathbf{x}_t)$ for playing vector $\mathbf{x}_t$. The function $f_t(\mathbf{x}): {\cal X} \to \mathbb{R}_+$ is defined by the adversary and can vary in each round, but it is assumed to be convex. We say that a function $f_t : \Real^d \mapsto \Real_+$ is {\it strongly convex} with modulus $\sigma>0$ if  
\begin{equation}
f_t(\mathbf{x}) \leq f_t(\mathbf{y}) + \nabla f_t(\mathbf{x})^\top(\mathbf{x}-\mathbf{y}) - \frac{\sigma}{2}\|\mathbf{x}-\mathbf{y}\|^2.\nonumber
\end{equation}
for any $\mathbf{x},\mathbf{y} \in \Real^d$.  We use the notation $\nabla f_t(\xb)$ to refer to any (sub-)gradient of $f_t$ at $\xb$. Furthermore, we say that $f_t$ is {\it convex} if $\sigma=0$.

\indent The learner's objective is to generate a sequence of vectors $\mathbf{x}_t \in {\cal X}$ for $t = 1,2,\cdots,T$ that minimizes the cumulative regret over $T$ rounds relative to the optimal vector $\mathbf{x}^\star$:
\begin{equation}
\!Regret_T(\mathbf{x}^*)\! \defin \!\sum_{t=1}^T f_t(\mathbf{x}_t) - \sum_{t=1}^T f_t(\mathbf{x}^\star). \hspace*{-0.05cm}\label{cumRegret}
\end{equation}
The latter measures the difference between the cumulative loss of the learner's sequence of vectors $\{\mathbf{x}_t\}_{t=1}^{T}$ and the accumulated loss that would be incurred if the sequence of loss functions $f_t$ would be known in advance and the learner could choose the best vector $\mathbf{x}^\star$ in hindsight.  

\indent Several algorithms have been developed over the past decade that achieve sub-linear cumulative regret in the OCO setting. The problem was formalized in the seminal work of \cite{Zinkevich2003}, which presents an online algorithm based on projected gradient descent \cite{Bertsekas1989} that guarantees a cumulative regret of ${\cal O}(T^{1/2})$ when the set ${\cal X}$ is convex and the loss functions are Lipschitz-continuous over ${\cal X}$. In \cite{Hazan2007} and \cite{Shalev-Shwartz2009}, algorithms with logarithmic regret bounds were proposed for strongly convex loss functions. Notably, online gradient descent achieves an ${\cal O}(\log T)$ regret bound for strongly convex loss functions for appropriate choices of step size. 

\indent In the aforementioned work, the constraint on vector $\mathbf{x}_t$ is assumed to hold in round $t$, such that a projection step is applied in every round to enforce the feasibility of each $\mathbf{x}_t$.  For general convex sets ${\cal X}$, the projection step may require solving an auxiliary optimization problem, which can be computationally expensive (e.g., projections onto the semi-definite cone). More importantly, in practical applications, the learner may in fact only be concerned with satisfying long-term constraints, that is, the {\it cumulative} constraint violations resulting from the sequence of vectors $\{\mathbf{x}_t\}_{t=1}^{T}$ should not exceed a certain amount by the final round $T$. An example of such an application is in online display advertising, where $\mathbf{x}_t$ corresponds to a vector of ad budget allocations and the learner is primarily concerned in enforcing the long-term constraint that each ad consumes its budget in full over the lifetime of the ad. Another example is from wireless communications \cite{Mannor2006}, where $\mathbf{x}_t$ is a vector of power allocations across multiple devices, and the learner must satisfy average power consumption constraints per device. 

\indent In this work, we consider OCO problems where the learner is required to satisfy long-term constraints. For such problems, we will be concerned both with a) the learner's cumulative regret as defined by \eqref{cumRegret} and b) the learner's ability to satisfy long-term constraints -- the notion of \textit{long-term} shall be made more formal in Section \ref{Sect: OCO}.
This class of problems was studied previously in \cite{Mahdavi2012,Mahdavi2012b}. In particular, \cite{Mahdavi2012b} considered online exponentially-weighted average in the case where the loss and constraints are linear, while \cite{Mahdavi2012} presented online algorithms based on projected subgradients and the mirror prox method \cite{Nemirovski2004}. The authors derived cumulative regret bounds for the cumulative loss and cumulative constraint violations respectively of ${\cal O}(T^{1/2})$ and ${\cal O}(T^{3/4})$ in the case of online projected subgradients, and respectively of ${\cal O}(T^{2/3})$ and ${\cal O}(T^{2/3})$ in the case of mirror prox. To our knowledge, these are the best-known regret bounds for OCO with long-term constraints. The analysis of \cite{Mahdavi2012} assumes the number of rounds is known ahead of time, which enables the authors to set the various constants in the algorithm that lead to the desired regret bounds. For the mirror prox method, it is additionally required that the constraint set ${\cal X}$ is fully specified by a finite number of linear constraints.

\indent The concept of long-term constraints also enables us to avoid the computation of potentially expensive projections onto the domain $\Xcal$ of $\xb$ in each round. This is closely related to recent work in stochastic optimization that aims to minimize the number of expensive projection steps~\cite{Mahdavi2012a}. The guarantees sought in our analysis have also similarities with the results obtained in the context of the online alternating direction method~\cite{Wang2013a}, where regret bounds are provided for the violation of equality constraints.

\textbf{Contributions:} Building on the work of \cite{Mahdavi2012}, we propose an algorithm based on a saddle-point formulation of the OCO problem with long-term constraints, which is adaptive (i.e., the step sizes and the regularisation parameter depend on the round $t$). We show that the algorithm satisfies cumulative regret bounds of ${\cal O}(T^{2/3})$ for both the cumulative loss and the cumulative constraint violations {\it without} requiring us to know the number of rounds $T$ in advance to set the algorithmic hyperparameters. Unlike mirror prox, our method can deal with any convex constraint set ${\cal X}$, making it amenable to a wider range of OCO problems. Also, the algorithm we derive allows us to interpolate between the regret bounds of ${\cal O}(T^{1/2})$ and ${\cal O}(T^{3/4})$ from \cite{Mahdavi2012} and the above bound of ${\cal O}(T^{2/3})$, depending on how we wish to trade off between the cumulative regret associated to the loss and the long-term constraints. In addition to our analysis of regret bounds, we empirically validate our algorithm by comparing it to the methods of \cite{Mahdavi2012} on a) the online estimation of doubly stochastic matrices and b) the online learning of sparse logistic regression based on the elastic net penalty~\cite{Zou2005}. \vspace{-0.1in}

\section{Online Convex Optimization with Long-term Constraints\vspace{-0.05in}}

\subsection{Preliminaries\vspace{-0.1in}}

\indent Consider $m$ convex functions $g_j : \Real^d \mapsto \Real$ which induce a convex constraint set\vspace{-0.05in}
$$
\Xcal \defin \Big\{  \xb \in \Real^d :   \max_{j \in \SET{m}} g_j(\xb) \leq 0 \Big\}.
$$
We assume that the set $\Xcal$ is bounded so that it is included in some Euclidean ball $\Bcal$ with radius $R >0$ (to be further discussed in Section~\ref{sec:discussion}):\vspace{-0.01in}
$$
\Xcal \subseteq \Bcal \defin \Big\{  \xb \in \Real^d :   \| \xb \|_2 \leq R \Big\}.
$$
Along with the functions $g_j$, we consider a sequence of convex functions $f_t :  \Real^d \mapsto \Real_+$ such that  
$$
\max_{t\in\SET{T}}\max_{\xb,\xb' \in \Bcal} |f_t(\xb) - f_t(\xb')| \leq F\,\text{ for some }F > 0.
$$
\indent  As is typically assumed in online learning~\cite{Cesa-Bianchi2006}, the functions $g_j$ and $f_t$ shall be taken to be Lipschitz continuous. We do not generally assume $g_j$ and $f_t$ to be differentiable. We further assume that for some finite $G > 0$ the (sub-)gradients of $f$ and $g_j$ are bounded
\begin{align}
\max_{j \in \SET{m}} \max_{\xb \in \Bcal}  \| \nabla g_j(\xb) \|_2 \leq G,\nonumber\\
\max_{t\in \SET{T}} \max_{\xb \in \Bcal}  \| \nabla f_t(\xb) \|_2 \leq G.\nonumber
\end{align} 
We take the same constant $G$ for, both, $g_j$ and $f_t$ for simplicity as we can always take the maximum between that of $g_j$ and that of $f_t$. Finally, we assume that there exists a finite $D>0$ such that the constraint functions are bounded over $\Bcal$:
$$
\max_{j \in \SET{m}} \max_{\xb \in \Bcal} |g_j(\xb)| \leq D.
$$
Finally, we note that the set of assumptions enumerated in this section are equivalent to the ones in~\cite{Mahdavi2012}.
\vspace{-0.1in}

\subsection{Problem Statement\vspace{-0.1in}}\label{Sect: OCO}
\indent Let $\{\xb_t\}_{t=1}^T$ be the sequence of vectors played by the learner and $\{f_t(\xb_t)\}_{t=1}^T$ the corresponding sequence of incurred losses. We aim at efficiently solving the following optimization problem in online fashion:\vspace{-0.05in}
\begin{align*}
&\min_{\xb_T\in\Bcal} \Big\{ f_T(\xb_T) + \dots \\
&+ \min_{\xb_2\in\Bcal} \Big\{ f_2(\xb_2) + \min_{\xb_1\in\Bcal} f_1(\xb_1) \Big\} \Big\} 
-
\min_{\xb \in \Xcal} \sum_{t=1}^T f_t(\xb)
\end{align*}
subject to the following \textit{long-term constraint}:\vspace{-0.05in}
$$
\max_{j \in \SET{m}} \sum_{t=1}^T g_j(\xb_t) \leq 0 .
$$

\section{Adaptive Online Algorithms based on a Saddle-point Formulation\vspace{-0.05in}}\label{Sect: Saddle}
\indent Following~\cite{Mahdavi2012}, we consider a saddle-point formulation of the optimization problem. For any $ \lambda \in \Real^+, \xb \in \Bcal$ we define the following function:
\begin{equation}\label{eq:lagrangianL2}
\Lcal_t(\xb,\lambda) \defin 
f_t(\xb) + \lambda g(\xb) - \frac{\theta_t}{2} \lambda^2 \nonumber
\end{equation}
where $g(\xb) \defin \max_{j\in\SET{m}} g_j(\xb)$ and $\{\theta_t\}_{t=1}^T$ is a sequence of positive numbers to be specified later. The role of $g$ is to aggregate the $m$ constraints into a single function. It otherwise preserves the same properties as those of individual $g_j$'s (sub-differentiability, bounded (sub-)gradients and bounded values; see Proposition 6 in \cite{Mahdavi2012} or Section 2.3 in \cite{Borwein2006} for a proof).
\indent In the saddle-point formulation~(\ref{eq:lagrangianL2}), we will alternate between minimizing with respect to the primal variable $\xb$ and maximizing with respect to the dual parameter $\lambda$.~A closer look at the function $\lambda \mapsto \Lcal_t(\xb,\lambda)$ indicates that we penalize the violation of the constraint $g(\xb) \leq 0$ through, using $[u]_+ \defin \max\{0,u\}$,  
\begin{eqnarray}\label{eq:squared_penalty}
\frac{1}{2 \theta_t}  [  g(\xb)  ]_+^2 = \sup_{\lambda \in \Real_+} \Big[   \lambda g(\xb) - \frac{\theta_t}{2} \lambda^2   \Big],
\end{eqnarray}
which is a penalty commonly used when going from constrained to unconstrained problems~(e.g., Section 17.1 in \cite{Nocedal2006}). Also, we can see from~(\ref{eq:squared_penalty}) that $\theta_t$ acts as a regularization parameter. We note at this point that in contrast to our method, \cite{Mahdavi2012} make use of a single $\theta$ that is constant in all rounds.\footnote{More precisely, in~\cite{Mahdavi2012}, $\theta$ is equal to the product of a constant step size times a constant scaling factor.}

\indent In the sequel, we study the following online algorithm where we alternate between primal-descent and dual-ascent steps:\vspace{-0.05in}
\begin{itemize}
    \setlength{\itemsep}{0pt}
    \setlength{\parskip}{0pt}
    \setlength{\parsep}{0pt} 
\item Initialize $\xb_1 = \zerob$ and $\lambda_1 = 0$
\item For $t \in \SET{T-1}$:
	\begin{itemize}
	\item[] $\xb_{t+1} = \Pi_{\Bcal}(  \xb_t - \eta_t  \nabla_\xb \Lcal_t(\xb_t,\lambda_t) )$
	\item[] $\lambda_{t+1} = \Pi_{\Real^+}(  \lambda_t + \mu_t  \nabla_\lambda \Lcal_t(\xb_t,\lambda_t) )$,
	\end{itemize}
\end{itemize} where  $\Pi_{\Ccal}$ stands for the Euclidean projection onto the set $\Ccal$, while $\{\eta_t\}_{t=1}^T$ and $\{\mu_t\}_{t=1}^T$ are sequences of non-negative step sizes that respectively drive the update of $\xb$ and $\lambda$. The algorithm is in the same vein as the ones proposed in~\cite{Mahdavi2012,Koppel2014}, but it is adaptive. The step sizes, which are different for the updates of $\xb$ and $\lambda$, are listed in Table~\ref{tab:params}.  They result from the analysis we provide in the next section. We also derive sub-linear regret bounds associated to these instantiations of the sequences $\{\theta_t\}_{t=1}^T$, $\{\eta_t\}_{t=1}^T$ and $\{\mu_t\}_{t=1}^T$.

\subsection{Main Results\vspace{-0.1in}}\label{Sect: main results}
\indent We begin by listing three sufficient conditions for obtaining sub-linear regret bounds for the proposed algorithm:\vspace{-0.1in}
\begin{itemize}[leftmargin=0.5in]
\item[\textbf{(C1):}] For any $t \geq 2$, $\frac{1}{\mu_t} -  \frac{1}{\mu_{t-1}} -  \theta_t \leq 0.$
\;\item[\textbf{(C2):}] For any $t \geq 2$, $\eta_t G^2 + \mu_t \theta_t^2 - \frac{1}{2} \theta_t \leq 0.$
\;\item[\textbf{(C3):}] For some finite $U_\eta > 0$, \\$\sum_{t=2}^T \Big[\frac{1}{\eta_t} -  \frac{1}{\eta_{t-1}} -  \sigma \Big] \leq U_\eta.$
\end{itemize}
Conditions \textbf{C1} and \textbf{C3} impose constraints on the decreasing speed of the step sizes.  We note that there is an asymmetry between $\mu_t$ and $\eta_t$: while we will always be able to control the norm of the variables $\xb_t$ (by design, they must lie in $\Bcal$), the sequence $\{\lambda_t\}_{t=1}^T$ is not directly upper-bounded in the absence of further assumptions on the gradient of $g$, hence the most stringent condition \textbf{C1} is to avoid any dependencies on $\lambda_t$.
Condition \textbf{C2} couples the behaviour of the three sequences to guarantee their validity. Finally, \textbf{C1, C2} and \textbf{ C3} are expressed for $t \geq 2$ because of our choice for the initial conditions $\xb_1=\zerob$ and $\lambda_1 = 0$. 

\indent Our main result is described in the following theorem, whose proof is in Section \ref{sec:analysis}:
\begin{theorem}\label{thm1}
Consider the convex case  ($\sigma = 0$) and the corresponding instantiations of the sequences $\mu_t, \eta_t$ and $\theta_t$ for some $\beta \in (0,1)$, as summarized in Table~\ref{tab:params}.
It holds for any $T \geq 1$ that 
$$
\sum_{t=1}^T \Delta f_t \! \leq \!
\Rcal_{T}^f  \!\defin\! 
 \!
\bigg[  RG + \frac{D^2}{6\beta RG} \bigg] T^\beta + \frac{2RG}{1-\beta} T^{1-\beta}
\!
$$
where $\Delta f_t \defin f_t(\xb_t) - f_t(\xb^\star)$, and
$$
\sum_{t=1}^T g(\xb_t) 
\leq 
\sqrt{
\frac{24RG}{1-\beta} \Big(\Rcal_{T}^f+ F T\Big) T^{1-\beta}
}.
$$
\end{theorem}
For the strongly convex case  ($\sigma>0$), we also have valid instantiations of $\eta_t, \mu_t, \theta_t$ (see Table~\ref{tab:params}), whose resulting cumulative regret bounds are tighter than those given in Theorem \ref{thm1}, but with the same leading terms.
Theorem \ref{thm1} can be stated in a simplified form, forgetting momentarily about the dependencies on $\{D,G,R,F\}$:
\begin{eqnarray*}
&\sum_{t=1}^T \Delta f_t \leq \Ocal(\max\{ T^{\beta},T^{1-\beta} \}),\\
&\sum_{t=1}^T g(\xb_t)  \leq \Ocal(T^{1-\beta/2}).
\end{eqnarray*} In particular, by setting $\beta = 2/3$, we obtain 
$$
\sum_{t=1}^T \Delta f_t \leq \Ocal(T^{2/3})\text{  and  }
 \sum_{t=1}^T g(\xb_t)  \leq \Ocal(T^{2/3}),
$$
which matches the mirror prox guarantees of~\cite{Mahdavi2012} while being valid for general convex constraint sets $\Xcal$ as opposed to just polyhedral constraint sets. Similarly, taking $\beta = 1/2$, we recover the regret bounds
$$
\sum_{t=1}^T \Delta f_t \leq \Ocal(T^{1/2})\text{  and  }
\sum_{t=1}^T g(\xb_t)  \leq \Ocal(T^{3/4}),
$$
of Section 3.1 in~\cite{Mahdavi2012}. We can, of course, also define novel trade-offs between loss and constraint violations, e.g., $\beta = 3/4$ with regret bounds of $\Ocal(T^{3/4})$ and $\Ocal(T^{5/8})$ respectively.
\begin{table}[]
\centering
\begin{tabular}{c|c|c}
 & Convex ($\sigma=0$) & Strongly convex ($\sigma>0$) \\ \hline\hline
$\theta_t$ &  $\frac{6RG}{t^\beta}$ & $\frac{6G^2}{\sigma t^\beta}$ \\
$\eta_t$ &  $\frac{R}{G t^\beta}$ & $\frac{1}{\sigma t}$  \\
 $\mu_t$ &  $\frac{1}{\theta_t (t + 1)}$ & $\frac{1}{\theta_t (t + 1)}$ \\\hline
$S_\theta$ &  $\frac{6RG}{1-  \beta} T^{1-\beta}$ &  $\frac{6G^2}{\sigma (1-\beta)} T^{1-\beta}$ \\
$S_\eta$ &  $\frac{R}{G (1-\beta)} T^{1-\beta}$ &  $\frac{1}{\sigma} (1+\log(T))$  \\
$S_\mu$ &  $\frac{1}{6  \beta  R G} T^\beta$ & $\frac{\sigma}{6  \beta  G^2} T^\beta$    \\
$U_\eta$ & $\frac{G}{R} T^\beta$  &  0\\
$\frac{1}{\mu_1}-\theta_1$ & $6RG$ & $\frac{6G^2}{\sigma}$  \\
$\frac{1}{\eta_1}-\sigma$ & $\frac{G}{R}$ & 0 
\end{tabular}
\caption{Parameter instantiations in different regimes ($\beta$ denotes some number in $(0,1)$).}
\label{tab:params}
\end{table}
\subsection{Analysis and Proofs}\label{sec:analysis}
\indent To analyze the above algorithm, we first introduce a series of lemmas. The analysis is analogous to that developed in~\cite{Mahdavi2012}, which we provide below for self-containedness. We begin by upper-bounding the variations of $\Lcal_t$ with respect to its two arguments. In particular, the following lemma takes advantage of the fact that the partial function $\lambda \mapsto \Lcal_t(\xb_t,\lambda)$ is not only concave as considered in~\cite{Mahdavi2012}, but \textit{strongly} concave with parameter $\theta_t$. This observation is at the origin of our improved regret bounds.
\begin{lemma}[Upper bound of $\Lcal_t(\xb_t,\lambda) - \Lcal_t(\xb_t,\lambda_t)$]\label{lmm1} For $\Lcal_t(\xb,\lambda)$ as defined above and for non-negative $\eta_t, \theta_t$ and $\mu_t$, we have
\begin{eqnarray*}
\Lcal_t(\xb_t,\lambda) - \Lcal_t(\xb_t,\lambda_t) \leq &\frac{1}{2\mu_t} \big[ b_t -   b_{t+1} \big] - \frac{\theta_t}{2} b_t \\ &+ \frac{\mu_t}{2} [ \nabla_\lambda \Lcal_t(\xb_t,\lambda_t) ]^2
\end{eqnarray*}
where $ b_t \defin ( \lambda - \lambda_t )^2$.
\end{lemma}
\begin{proof}
Expanding $(\lambda - \lambda_{t+1})^2$ yields
\begin{align}
(\lambda - \lambda_{t+1})^2 &= \Big(\lambda - \Pi_{\mathbb{R}^+}(\lambda_{t} + \mu_t \nabla_\lambda {\cal L}_t (\mathbf{x}_t,\lambda_t))\Big)^2 \nonumber\\ &\leq \Big(\lambda - (\lambda_{t} + \mu_t \nabla_\lambda {\cal L}_t (\mathbf{x}_t,\lambda_t))\Big)^2\nonumber \\&=(\lambda - \lambda_{t})^2 -2\mu_t (\lambda - \lambda_{t})\nabla_\lambda {\cal L}_t (\mathbf{x}_t,\lambda_t) \nonumber \\ &+ \mu_t^2 (\nabla_\lambda {\cal L}_t (\mathbf{x}_t,\lambda_t))^2\nonumber
\end{align} By strong concavity of  ${\cal L}_t(\mathbf{x}_t,\lambda)$ with respect to $\lambda$, 
\begin{align}
{\cal L}_t(\mathbf{x}_t,\lambda_t) - {\cal L}_t(\hat{\mathbf{x}}_t,\lambda) &\leq  (\lambda - \lambda_{t})\nabla_\lambda {\cal L}_t (\mathbf{x}_t,\lambda_t) - \frac{\theta_t}{2}b_t. \nonumber
\end{align}
Substituting the inequality for $\mu_t (\lambda - \lambda_{t})\nabla_\lambda {\cal L}_t (\mathbf{x}_t,\lambda_t)$ completes the proof.
\end{proof}
We omit the derivation for $\xb \mapsto \Lcal_t(\xb,\lambda_t)$ that follows similar arguments. We now turn to a lower-bound of the variations of $\Lcal_t$.
\begin{lemma}\label{lmm2} Let $\xb^\star=\arg \min_{\xb \in \Xcal} \sum_{t=1}^T f_t(\xb)$. Then 
\begin{eqnarray*}
\sum_{t=1}^T \Lcal_t(\xb_t,\lambda) - \Lcal_t(\xb^\star,\lambda_t) \geq \\
\sum_{t=1}^T \Delta f_t +\lambda \sum_{t=1}^T g(\xb_t) \nonumber - \frac{\lambda^2}{2} \sum_{t=1}^T \theta_t
+\frac{1}{2} \sum_{t=1}^T \theta_t \lambda_t^2.
\end{eqnarray*}
\end{lemma}
\begin{proof} We have ${\Lcal}_t(\mathbf{x}_t,\lambda) - {\Lcal}_t(\mathbf{x}^\star,\lambda_t)$ equal to
\begin{eqnarray*}
f_t(\mathbf{x}_t) - f_t(\mathbf{x}^\star) + \lambda g(\mathbf{x}_t) \nonumber - \lambda_t g(\mathbf{x}^\star) - \frac{\theta_t}{2}(\lambda^2-\lambda_t^2) .
\end{eqnarray*} 
We simply notice that $g(\xb^\star) \leq 0$ to obtain a lower bound $- g(\xb^\star) \sum_{t=1}^T \lambda_t$ to complete the proof after summing both sides over rounds $t=1,\cdots,T$.
\end{proof} 
\begin{lemma}\label{lmm3} Let $a_t \defin \|\mathbf{x}_t-\mathbf{x}\|^2$. For any $\sigma, \theta_t \geq 0$, \vspace{-0.05in}
\begin{eqnarray*}
\sum_{t=1}^T \frac{1}{2\eta_t} \big[ a_t -   a_{t+1} \big] - \frac{\sigma}{2} a_t
\leq \\
\frac{R^2}{2} \delta_\eta
+ 
\frac{1}{2} \sum_{t=2}^T a_t  \bigg[ \frac{1}{\eta_t} -  \frac{1}{\eta_{t-1}} -  \sigma \bigg],
\\
\sum_{t=1}^T \frac{1}{2\mu_t} \big[ b_t -   b_{t+1} \big] - \frac{\theta_t}{2} b_t
\leq \\
\frac{\lambda^2}{2}  \delta_\mu
+
\frac{1}{2} \sum_{t=2}^T b_t  \bigg[ \frac{1}{\mu_t} -  \frac{1}{\mu_{t-1}} -  \theta_t \bigg]
\end{eqnarray*}
where we have used $\delta_\eta \defin \frac{1}{\eta_1} -   \sigma$ and $\delta_\mu \defin \frac{1}{\mu_1} -   \theta_1$.
\end{lemma}\vspace{-0.1in}
\begin{proof}
Shifting indices in the sums for terms that depend on $a_{t+1}/\eta_t, b_{t+1}/\mu_t$ and collecting terms that depend on $a_t,b_t$, we then use
$a_1 = \|\xb_1 -\xb \|_2^2 = \|\xb \|_2^2 \leq R^2$ and 
$b_1 = (\lambda - \lambda_1)^2 = \lambda^2$ to conclude.
\end{proof} \vspace{-0.05in} 
We now present the key lemma of the analysis.
\begin{lemma} \label{lmm5} [Cumulative regret bound] Let $\xb^\star = \arg \min_{\xb \in \Xcal} \sum_{t=1}^T f_t(\xb)$ and assume \textbf{C1}, \textbf{C2} and \textbf{C3} hold. 
Define 
${\cal R}_T^f \defin \frac{R^2}{2} \delta_\eta
+
G^2 S_\eta
+
D^2 S_\mu
+
R^2 U_\eta$, 
where we have introduced $S_\eta \defin \sum_{t=1}^T \eta_t$, $S_\mu \defin \sum_{t=1}^T \mu_t$ and $S_\theta \defin \sum_{t=1}^T \theta_t$.
Then, it holds that
\begin{eqnarray*}
&\sum_{t=1}^T \Delta f_t
\leq
{\cal R}_T^f,\\
&\text{  and  }\\
&\sum_{t=1}^T g(\xb_t) 
\leq
\sqrt{ 2(S_\theta + \delta_\mu)
({\cal R}_T^f + FT)}.
\end{eqnarray*}
\end{lemma}
\begin{proof} By the triangle inequality, we have $\| \nabla_\xb \Lcal_t(\xb,\lambda)\|_2^2 \leq 2 G^2 (1 + \lambda^2)$ and $\big(\nabla_\lambda \Lcal_t(\xb,\lambda) \big)^2 \leq 2 (D^2 + \theta_t^2 \lambda^2)$. We then combine Lemmas \ref{lmm1}-\ref{lmm3}, starting from $ \Lcal_t(\xb_t,\lambda) - \Lcal_t(\xb^*,\lambda_t) =  \Lcal_t(\xb_t,\lambda) - \Lcal_t(\xb_t,\lambda_t) +  \Lcal_t(\xb_t,\lambda_t) - \Lcal_t(\xb^*,\lambda_t)$, yielding\vspace{-0.05in}
\begin{eqnarray*}
&\sum_{t=1}^T \Delta f_t
+\lambda \sum_{t=1}^T g(\xb_t) - \frac{\lambda^2}{2} \bigg[S_\theta +  \delta_\mu \bigg]
\leq\\
&
{\cal R}_T^f
+
 \sum_{t=1}^T \lambda_t^2 \bigg[   \eta_t G^2 + \mu_t \theta_t^2 - \frac{1}{2} \theta_t \bigg].
\end{eqnarray*}
Maximizing the left-hand side with respect to $\lambda \in \Real^+$, we obtain:
\begin{eqnarray*}
&\sum_{t=1}^T \Delta f_t
+\frac{ \big[\sum_{t=1}^T g(\xb_t)\big]_+^2}{2 \big[S_\theta + \delta_\mu \big]}
\leq\\
&{\cal R}_T^f
+
 \sum_{t=1}^T \lambda_t^2 \bigg[   \eta_t G^2 + \mu_t \theta_t^2 - \frac{1}{2} \theta_t \bigg].\vspace{-0.3in}
\end{eqnarray*} The regret bound on the loss is obtained by using \textbf{C2} and $[\sum_{t=1}^T g(\xb_t)]_+^2/(2 [S_\theta + \delta_\mu]) \geq 0$. The bound on constraint violations is obtained as above, but by substituting the lower bound $\sum_t \Delta f_t \geq -FT$.
\end{proof}\vspace{-0.1in}
In order to discuss the scaling of our regret bounds, we state the next simple lemma without proof
\begin{lemma}\label{lmm4} Let $\beta \in (0,1)$. Then
$
\sum_{t=1}^T \frac{1}{t^{\beta}} \leq \frac{T^{1-\beta}}{1-\beta}.
$
\end{lemma} 
With the above lemmas, we now prove Theorem \ref{thm1}:
\begin{proof}[Proof of Theorem 1]
\indent For the proposed choices of $\theta_t,\mu_t$ and $\eta_t$, we can verify that \textbf{C1}, \textbf{C2} and \textbf{C3} hold.
Here we focus on the convex case, the strongly convex one following along the same lines.
First, we can easily see that \textbf{C1} is true as long as $\theta_t$ is non-increasing.
Then, we can notice that, given the choice of $\mu_t$, condition \textbf{C2} is implied by the stronger condition $\eta_t \leq \frac{\theta_t}{6G^2}$ (satisfied by the choice of $\eta_t$ and $\theta_t$ in Table 1). This results in 
$$
S_\mu = \sum_{t=1}^T \frac{1}{\theta_t (t + 1)} \leq \sum_{t=1}^T \frac{t^\beta}{6RG t} \leq  \frac{T^\beta}{6\beta R G}.
$$
$$
S_\eta = \sum_{t=1}^T \eta_t \leq \frac{R}{G} \frac{T^{1-\beta}}{1-\beta}
\,,
\,S_\theta = \sum_{t=1}^T \frac{6RG}{t^\beta} \leq \frac{6RG}{1-\beta} T^{1-\beta},
$$
along with
$1/\mu_1 -   \theta_1 = 6RG$ and $1/\eta_1 -   \sigma = G/R.$
The term $U_\eta$ can be obtained by summing the series $1/\eta_t - 1/\eta_{t-1} = (G/R)(t^\beta - (t-1)^{\beta})$ over $t$, which directly simplifies by telescoping for the $\sigma=0$ case, and is identically equal to zero for $\sigma>0$.
As a result, we obtain from Lemma \ref{lmm5}
$$
\sum_{t=1}^T \Delta f_t \! \leq \!
\Rcal_{T}^f  \!\defin\! 
 \!
\bigg[  RG + \frac{D^2}{6\beta RG} \bigg] T^\beta + \frac{RG}{1-\beta} T^{1-\beta}  + \frac{RG}{2}
\!
$$
and for the constraint
$$
\sum_{t=1}^T g(\xb_t) 
\leq 
\sqrt{
2 (\Rcal_{T}^f+ FT)
\bigg[
\frac{6RG}{1-\beta} T^{1-\beta} + 6RG
\bigg]}.
$$
We obtain the desired conclusion by noticing that for any $T \geq 1$ and $\beta \in (0,1)$, we have $\frac{T^{1-\beta}}{1-\beta} \geq 1$.
\end{proof}

\subsection{Towards No Violation of Constraints}\vspace{-0.1in}
\indent We next show that our extension also applies to the more specific setting developed in Section 3.2 from \cite{Mahdavi2012}, where additional assumptions on the gradient of $g$ can translate into no constraint violations. 
For the sake of self-containedness, we briefly recall it.
Assume that there exist $\gamma \geq 0$ and $r > 0$ such that the variations of $g$ are lower bounded as
\begin{equation}\label{eq:lower_bound_gradient}
\min_{\xb \in \Real^d: g(\xb) + \gamma = 0} \|  \nabla g(\xb)  \|_2 \geq r.
\end{equation}
Denote $\Xcal_\gamma \defin \{\xb \in \Real^d : g(\xb) + \gamma \leq 0 \} \subset \Xcal$.
It can then be shown that (see Theorem 7 in \cite{Mahdavi2012}):
\begin{eqnarray}\label{eq:gamma_approx}
&\Big|  \sum_{t=1}^T f_t(\xb^\star) -   \sum_{t=1}^T f_t(\xb_\gamma)    \Big| \leq \frac{G}{r} T \gamma,
\end{eqnarray}
where $\xb^\star$ and $\xb_\gamma$ are solutions of $\min_{\xb \in \Xcal}  \sum_{t=1}^T f_t(\xb)$ and $\min_{\xb \in \Xcal_\gamma}  \sum_{t=1}^T f_t(\xb)$ respectively. In words, the gap between the optimal value of the original optimization problem and that of the problem over $\Xcal_\gamma$ is well-controlled as a function of $(\gamma, r)$.
Examples where~(\ref{eq:lower_bound_gradient}) holds include the positive semi-definite cone, as described in Section 4 of~\cite{Mahdavi2012a}. For space limitation reasons, we state our result in a simplified form, only briefly sketching its proof that follows the same logic as that of the previous section.
\begin{corollary} Assume~(\ref{eq:lower_bound_gradient}) holds.
Consider the convex case  ($\sigma = 0$) and some instantiations of the sequences $\mu_t, \eta_t$ and $\theta_t$ for some $\beta \in (0,1)$, differing from Table~\ref{tab:params} only up to numerical constants.
There exist $c_0$ and $c_1$ depending on $\{D,G,R,F,r\}$ such that setting $\gamma \defin c_1 T^{-\beta/2}$, we have for any $T \geq c_0$
\begin{eqnarray*}
&\sum_{t=1}^T \Delta f_t \leq \Ocal(\max\{ T^{\beta},T^{1-\beta}, T^{1 - \beta/2}\}),
\end{eqnarray*}
and no constraint violations $\sum_{t=1}^T g(\xb_t)  \leq 0.$
\end{corollary}
\begin{proof}[Sketch of proof]
We can apply the same analysis as before to the function $g_\gamma(\xb) \defin g(\xb) + \gamma$, replacing $D$ by $D+\gamma$ and adapting the constants in both \textbf{C2} (i.e., $\eta_t G^2 + \frac{3}{2}\mu_t \theta_t^2 - \frac{1}{2} \theta_t$) as well as for the instantiations of $\mu_t,\eta_t$ and $\theta_t$. The regret bound on $\sum_{t=1}^T \Delta f_t$ is identical  as earlier, with additional additive terms $3\gamma^2 S_\mu/2$ and $ G T \gamma/r$ introduced as a result of ~(\ref{eq:gamma_approx}). As for $\sum_{t=1}^T g(\xb_t)$, the term $[\sum_{t=1}^T g(\xb_t)]_+^2$ becomes here $[\sum_{t=1}^T g(\xb_t) + \gamma T]_+^2$, which in turn leads to the same regret bound as previously stated, minus the contribution $-\gamma T$. We cancel out the constraint violations---scaling in $\Ocal(T^{1-\beta/2})$ according to Theorem~\ref{thm1}---by setting $\gamma = c_1 T^{-\beta/2}$.
Note that $c_0$ is determined by examining when the extra term $3\gamma^2 S_\mu/2$ can be upper bounded by those in $\Rcal_{T}^f$.
\end{proof}
The regret bound presented in Corollary 1 is minimized for $\beta = 2/3$, leading to a regret of $\Ocal(T^{2/3})$ with no constraint violations. This result extends Theorem 8 and Corollary 13 from~\cite{Mahdavi2012} in that it holds for general convex domains $\Xcal$ (as opposed to only polyhedral ones).   \vspace{-0.1in}
\section{Experiments}\vspace{-0.1in}
\indent We ran two sets of experiments to validate the regret bounds obtained for our adaptive algorithms for OCO with long-term constraints and compare to the algorithms proposed in~\cite{Mahdavi2012}. First, we examine the online estimation of doubly-stochastic matrices where the convex domain of interest $\Xcal$ is polyhedral but whose projection operator is difficult to compute~\cite{Helmbold2009,Fogel2013}. Second, we consider sparse online binary classification based on the elastic net penalty~\cite{Zou2005}.

\indent We shall refer to our adaptive online gradient descent (A-OGD) for convex $f_t$ (i.e., $\sigma = 0$) as \texttt{Convex~A-OGD} and for strongly convex  $f_t$ (i.e., $\sigma > 0$) as \texttt{Strongly~convex~A-OGD}, which enjoy the same regret guarantees of $\Ocal(T^{2/3})$ for the loss and constraint. The method of~\cite{Mahdavi2012} that handles general convex domains $\Xcal$ will be referred to as \texttt{Convex~OGD}, while the mirror prox method analyzed in~\cite{Mahdavi2012}, which is only applicable to polyhedral domains, will be denoted by \texttt{Convex~mirror~prox}. The parameters of \texttt{Convex~OGD}  and \texttt{Convex~mirror~prox} are instantiated according to~\cite{Mahdavi2012}.

\indent Although we generate the sequence of losses $\{f_t\}_{t=1}^T$ stochastically in the experiments, we would like to emphasize that the regret bounds we obtain are also valid for adversarially-generated sequences, so that it is not required that $\{f_t\}_{t=1}^T$ are generated in i.i.d. fashion. \vspace{-0.1in}

\subsection{Doubly-Stochastic Matrices}\vspace{-0.1in}
\indent Doubly-stochastic matrices appear in many machine learning and optimization problems, such as clustering applications~\cite{Zass2006} or learning permutations~\cite{Helmbold2009,Fogel2013}. Briefly, for a sequence of matrices $\{\Yb_t\}_{t=1}^T$ in $\Real^{p\times p}$, we solve the following optimization problem in an online fashion:
\begin{eqnarray}\label{eq:dsm}
\min_{\Xb \in \Real^{p\times p}} \sum_{t=1}^T \frac{1}{2} \|  \Yb_t  - \Xb \|_\fro^2
\end{eqnarray}
subject to the (linear) convex constraints
$$
 \Xb \geq \zerob,\ \
 \Xb \oneb = \oneb \text{ and } \Xb^\top \oneb = \oneb.
$$
This problem can easily be mapped to OCO setting by assuming that the sequence $\{\mathbf{Y}\}_{t=1}^T$ is generated by random permutation matrices which are known to constitute the extreme points of the set of doubly-stochastic matrices~\cite{Birkhoff1946}. We have $d=p^2$, $f_t(\Xb) =  \frac{1}{2} \|  \Yb_t  - \Xb \|_\fro^2$ and $m = p^2 + 4p$ to describe all the linear constraints, more specifically there are $p^2$ non-negativity constraints, along with $4p$ inequalities to model the $2p$ equality constraints. This leads to the following instantiations of the parameters controlling $f_t$ and $g$: $R = \sqrt{p}$, $G=2R$ and $D=R$. Note that we can apply a) \texttt{Strongly~convex~A-OGD} (since $f_t$ is by construction strongly convex with parameter $\sigma = 1$), and b) \texttt{Convex~mirror~prox} since $\Xcal$ is polyhedral.

\indent The cumulative regret for the loss and the long-term constraint are shown in Figures~\ref{fig:ft_dsm} and~\ref{fig:gt_dsm}. They are computed over $T=1000$ iterations with $d = 64$, and are averaged over 10 random sequences $\{\Yb_t\}_{t=1}^T$ (the standard deviations are not shown since they are negligible). The offline solutions of~(\ref{eq:dsm}) required for various $t \in \SET{T}$ to compute the regret are obtained using \texttt{CVXPY}~\cite{Diamond2014}.
\begin{figure}[t]
  \centering
  \includegraphics[width=0.45\textwidth]{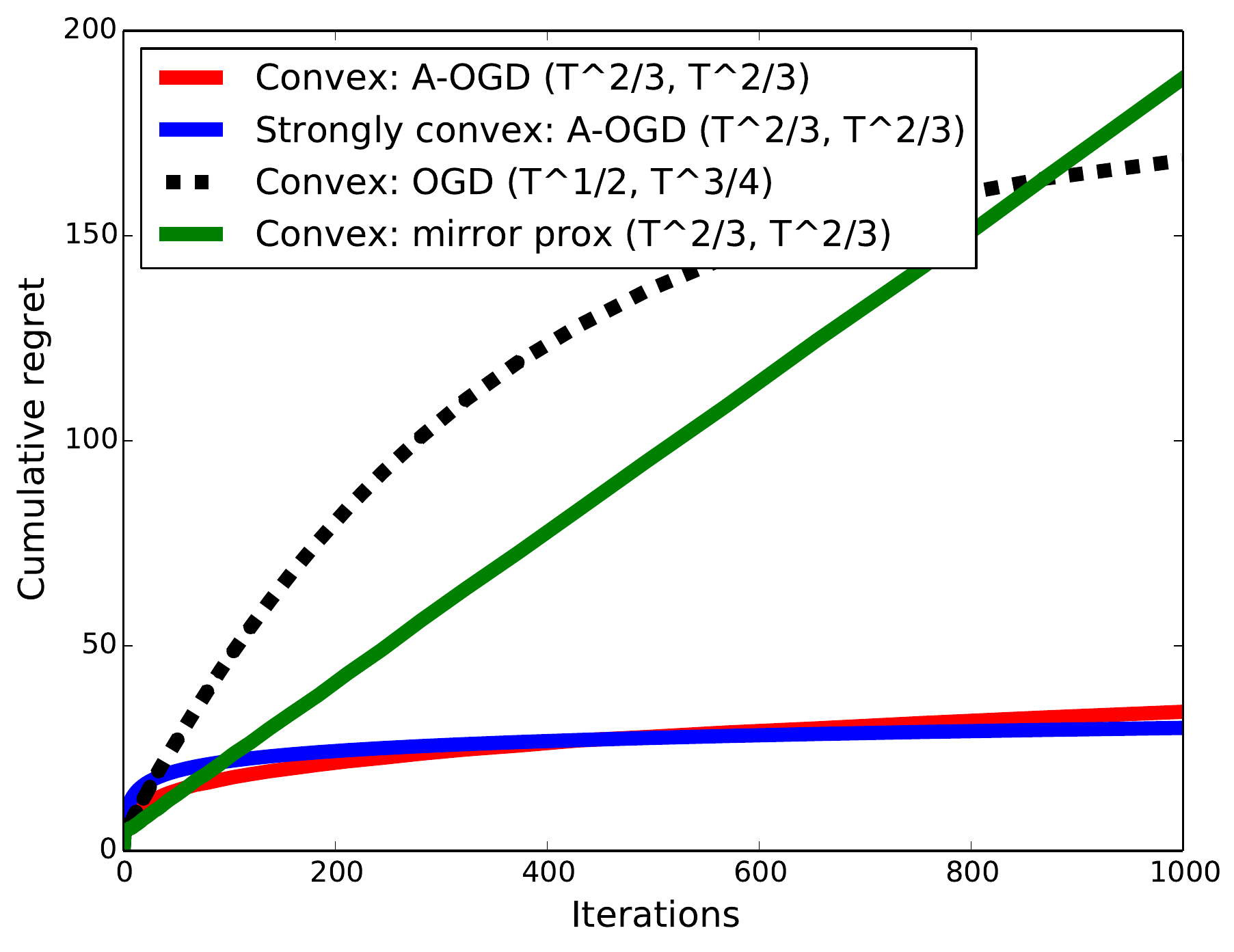}
    \caption{Cumulative regret of the loss function for the online estimation of doubly-stochastic matrices. We display the mean computed over 10 random sequences $\{\Yb_t\}_{t=1}^T$.~(Best viewed in colour.)}
    \label{fig:ft_dsm}
\end{figure}
\begin{figure}[t]
  \centering
   \includegraphics[width=0.45\textwidth]{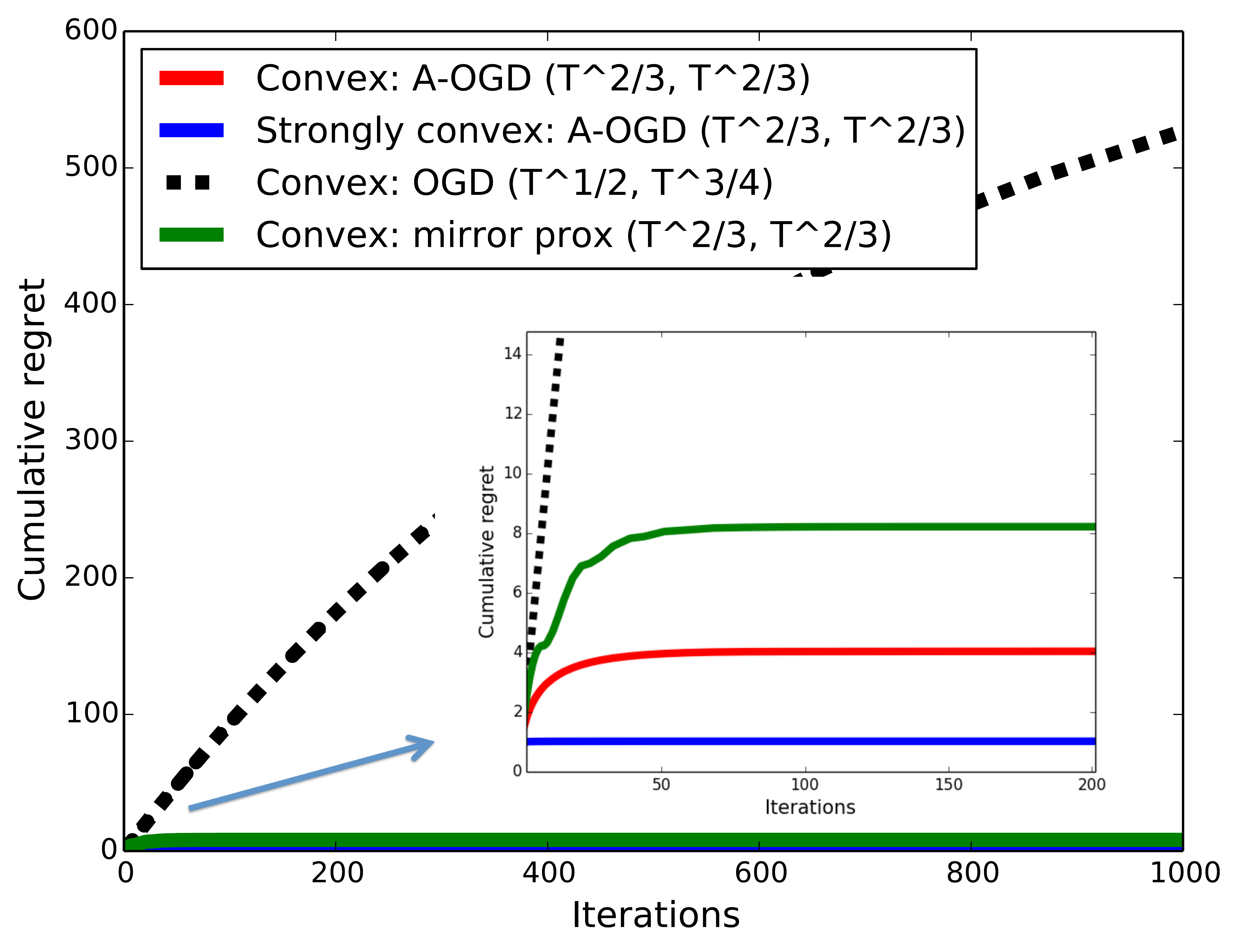}
    \caption{Cumulative regret of the long-term constraint for the estimation of doubly-stochastic matrices. We display the mean computed over 10 random sequences $\{\Yb_t\}_{t=1}^T$. The embedded graph is a zoom of the area of interest in the original figure.~(Best viewed in colour.)}
    \label{fig:gt_dsm}
\end{figure}

\indent The results shown in Figure~\ref{fig:ft_dsm} and~\ref{fig:gt_dsm} indicate that although the cumulative regret bounds for \texttt{Strongly~convex~A-OGD} were not demonstrated to be tighter in our analysis than those for \texttt{Convex~A-OGD}, they achieve a better cumulative regret for this problem, especially with respect to the long-term constraint. Also, while \texttt{Convex~mirror~prox} and \texttt{Convex~A-OGD} should theoretically exhibit the same behavior, the results suggest that mirror prox is not able to decrease cumulative regret at the same rate as our proposed method. We surmise that this may be due to the fact that the guarantees of \texttt{Convex~mirror~prox} proved in~\cite{Mahdavi2012} only hold for very large $T$.\footnote{Theorem 12 from~\cite{Mahdavi2012} requires $T \geq 164(m+1)^2$, which translates into $T > 10^7$ in our setting.}
\vspace{-0.1in}
\subsection{Sparse Online Binary Classification}\vspace{-0.1in}
\indent Next, we examine the application of sparse online binary classification. Our goal will be to minimize the log-loss subject to a constrained elastic-net penalty:\footnote{Constraint formulations with sparsity-inducing terms are sometimes preferred over their penalized counterparts when they express some concrete physical budget, e.g., in the context of learning predictors with low-latency~\cite{Xu2012}.}
\begin{eqnarray}\label{eq:binaryclassification}
\min_{\substack{\xb \in \Real^{d}: \|\xb\|_1 + \frac{1}{2}\|\xb\|_2^2 \leq \rho}} \sum_{t=1}^T \log( 1 + e^{ -y_t \xb^\top \ub_t } ),
\end{eqnarray}
where $\{y_t,\ub_t\}_{t=1}^T$ denotes a sequence of label/feature-vector pairs and $\rho > 0$ is a parameter that measures the degree of the sparsity of the solutions of~(\ref{eq:binaryclassification}).

\indent This problem is mapped to our formulation by setting $f_t(\xb) =  \log( 1 + e^{ -y_t \xb^\top \ub_t } )$ with 
$m = 1$, 
$g(\xb) = \|\xb\|_1 + \frac{1}{2}\|\xb\|_2^2 - \rho$,
$R = \sqrt{1+2\rho}-1$,\footnote{The value for $R$ is found by noticing that $\|\xb\|_1 + \frac{1}{2}\|\xb\|_2^2 \geq \|\xb\|_2 + \frac{1}{2}\|\xb\|_2^2$ and solving the resulting second-order polynomial inequality.} 
$G = \max\{\sqrt{d}+R,\max_{t} \|\ub_t\|_2\}$ and 
$D = \sqrt{d} R + R^2/2$. 
The sequences $\{y_t,\ub_t\}_{t=1}^T$ are generated by drawing pairs at random with replacement.

\indent We solve the above problem using the datasets \texttt{ijcnn1} and  \texttt{covtype}, consisting respectively of $49,990$ and $581,012$ samples of dimension $d=22$ and $d=54$ each\footnote{\scriptsize{\texttt{www.csie.ntu.edu.tw/~cjlin/libsvmtools/datasets/binary.html}}}.  The parameter $\rho$ is set to obtain approximately $30\%$ of non-zero variables. Moreover, and in order to best display cumulative regret, we compute offline solutions of~(\ref{eq:binaryclassification}) for various $t\in\SET{T}$ thanks to an implementation of~\cite{Defazio2014}.

\indent The results are summarized in Figures~\ref{fig:ft_bc_ijcnn1_covtype} and \ref{fig:gt_bc_ijcnn1_covtype} and represent an average over 10 random sequences $\{y_t,\ub_t\}_{t=1}^T$ (the standard deviations are not shown since they are negligible). The number of iterations $T$ is equal to the number of samples in each dataset. 
\begin{figure}[t]
  \centering
  \hspace*{-0.1cm}\includegraphics[width=0.5\textwidth]{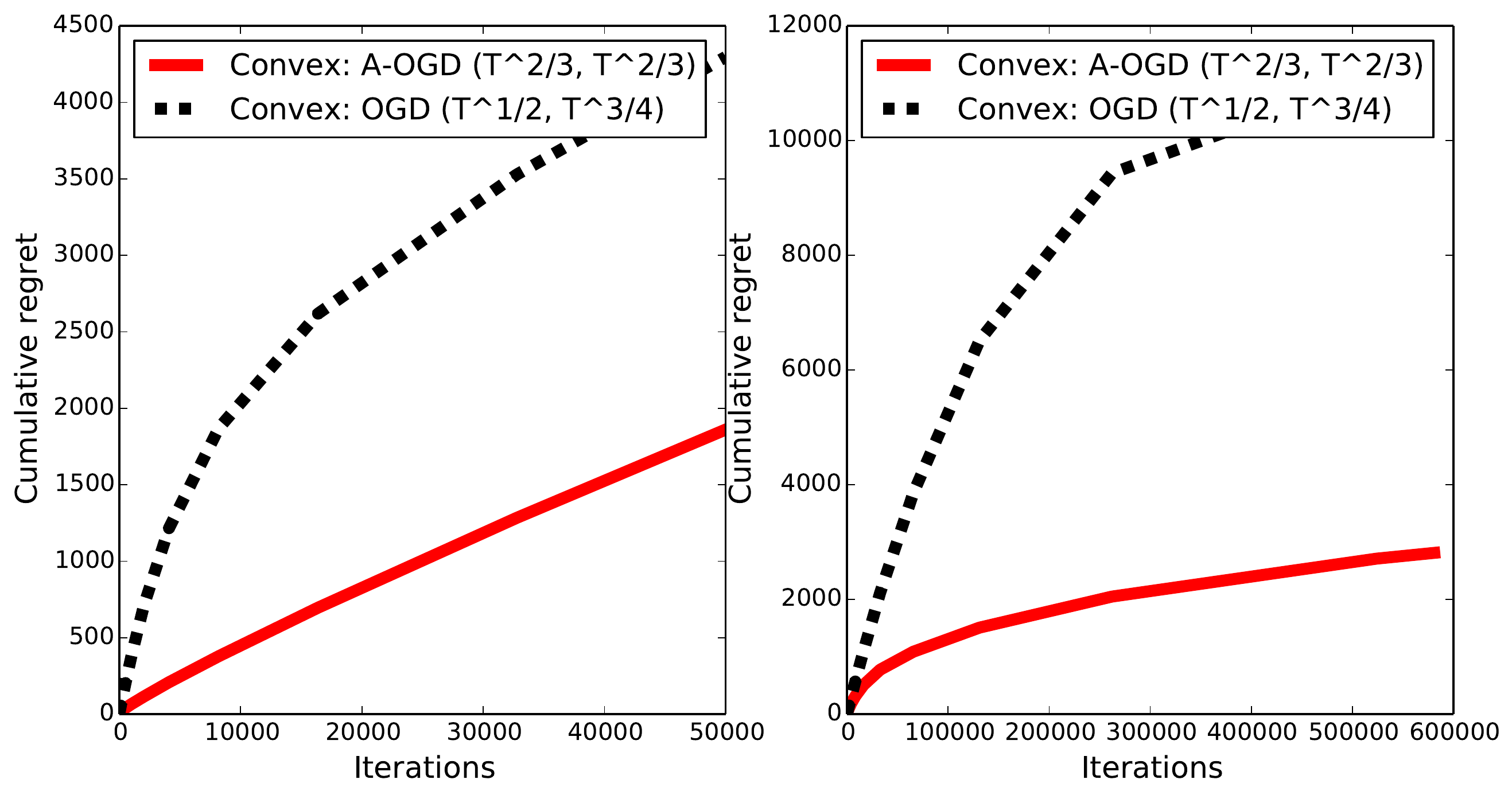}
    \caption{Cumulative regret of the loss function for the sparse online binary classification (left:~\texttt{ijcnn1}, right:~\texttt{covtype}). We display the mean computed over 10 random sequences $\{y_t,\ub_t\}_{t=1}^T$.}\vspace{-0.1in}
    \label{fig:ft_bc_ijcnn1_covtype}
\end{figure}
\indent Interestingly, we observe that the constraint is not violated on average (i.e., via a negative cumulative regret) and the iterates $\xb_t$ remain feasible within the domain $\|\xb\|_1 + \frac{1}{2}\|\xb\|_2^2 \leq \rho$. This tendency is more pronounced for \texttt{Convex~OGD}  since a closer inspection of the sequence $\{\eta_t\}_{t=1}^T$ shows numerical values smaller than those of our approach \texttt{Convex~A-OGD} (by 2 to 3 orders of magnitude). As a result, starting from $\xb_1 = \zerob$, we found that the iterates generated by \texttt{Convex~OGD}  do not approach the boundary of the domain, hence increasing regret on cumulative loss. We also note that the offline solutions of~(\ref{eq:binaryclassification}) always saturate the constraint. Although our analysis predicts that the cumulative regret of \texttt{Convex~OGD} associated to the loss (i.e., $\Ocal(T^{1/2})$) should be smaller than that associated to \texttt{Convex~A-OGD} (i.e., $\Ocal(T^{2/3})$), \texttt{Convex~A-OGD} achieves here a lower cumulative regret. This observation may be explained by the same argument as that described previously, namely that the larger step sizes $\{\eta_t\}_{t=1}^T$ of \texttt{Convex~A-OGD} enables us to make faster progress.
\begin{figure}[t]
  \centering
  \hspace*{-0.1cm}\includegraphics[width=0.5\textwidth]{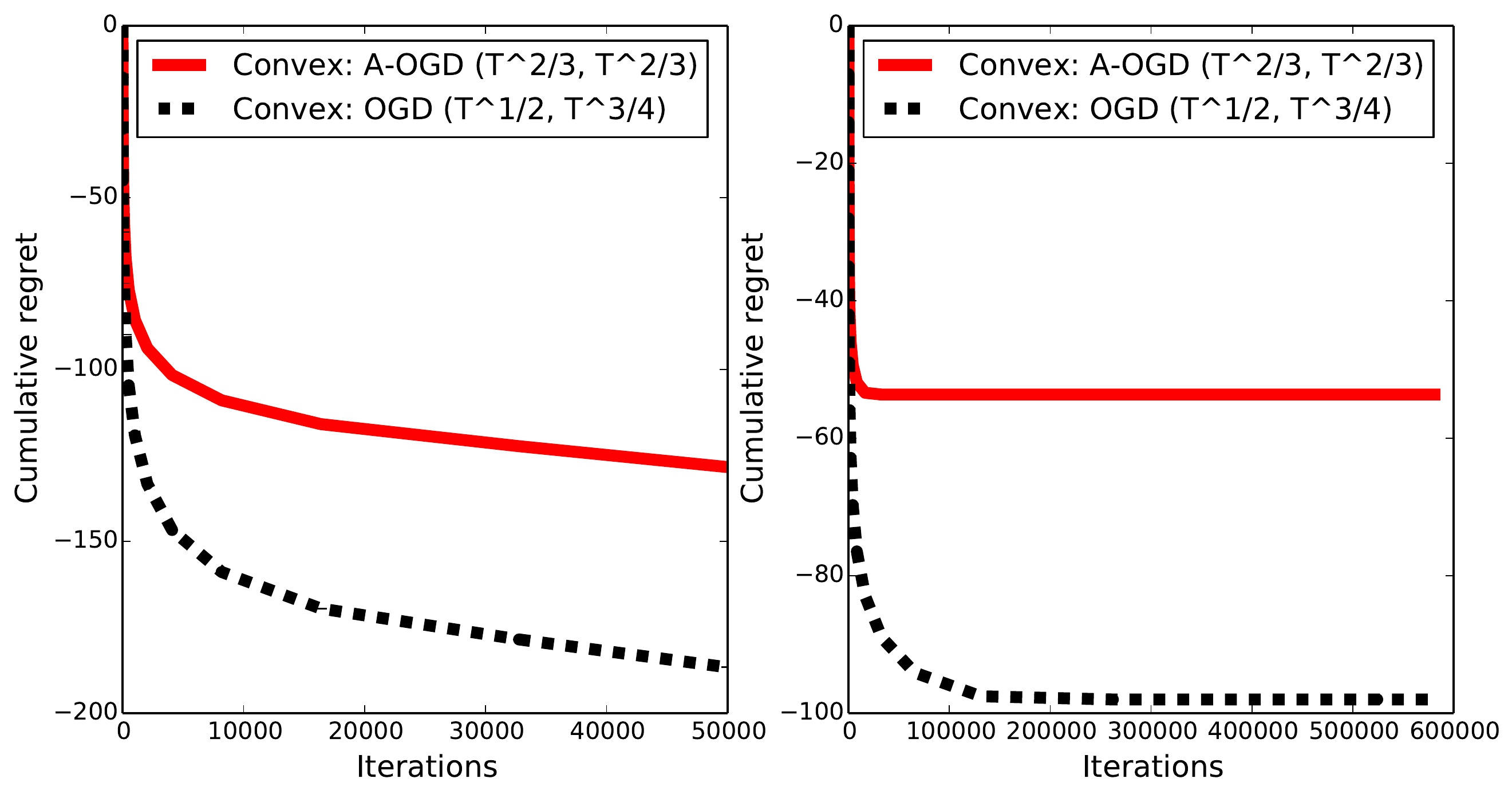}
    \caption{Cumulative regret of the constraint for the sparse online binary classification (left:~\texttt{ijcnn1}, right:~\texttt{covtype}). We display the mean computed over 10 random sequences $\{y_t,\ub_t\}_{t=1}^T$.}\vspace{-0.1in}
    \label{fig:gt_bc_ijcnn1_covtype}
\end{figure}\vspace{-0.1in}
 \section{Discussion\vspace{-0.1in}}\label{sec:discussion}
In this section, we discuss several generalizations.\vspace{-0.1in}

\paragraph{Broader families of step sizes.}
We have assumed that the updates of the primal variable $\xb$ are driven by a projected gradient step controlled through a single step size $\eta_t$. Following the ideas developed in~\cite{McMahan2010, Duchi2011a}, we could analyze the regret guarantees of our algorithm when there is a vector of step sizes $\boldsymbol{\eta}_t$ that is given by a diagonal matrix $\Diag(\etab_t) \in \Real^{d\times d}$, updating adaptively and separately each coordinate of $\xb$. 

\paragraph{Can we identify a better penalty?}
In the light of (\ref{eq:squared_penalty}), it is tempting to ask whether we can find a penalty function that will lead to lower cumulative regret guarantees. To this end, we could for example introduce a smooth, $1$-strongly-convex function $\phi$ with domain $\Omega$. The saddle-point formulation of the new problem is then given by
\begin{equation*}\label{eq:lagrangianGeneral}
\Lcal_t(\xb,\lambda) \defin 
f_t(\xb) + \lambda g(\xb) - \theta_t \phi(\lambda),
\end{equation*}
where $\{\theta_t\}_{t=1}^T$ is, as earlier, a sequence of non-negative numbers to be specified subsequently for any $\lambda \in \Omega, \xb \in \Bcal$. Interestingly, it can be shown that condition \textbf{C2} becomes a first-order nonlinear ordinary differential inequality in this setting, leading to
$$
\eta_t G^2 \lambda^2  + \mu_t \theta_t^2 \left[ \frac{d\phi}{d\lambda} \right]^2 - \theta_t \phi(\lambda) \leq 0, \text{  for all } \lambda \in \Omega.
$$
Hence, the above differential inequality suggests a family of penalty functions which we could use. In particular, we see that $\phi$ must grow at least quadratically and stay greater than its squared first derivative, which rules out a softmax penalty like $\lambda \mapsto \log(1+e^\lambda)$. Moreover, the maximization with respect to $\lambda$ in the last step of Lemma \ref{lmm5} introduces the Moreau envelope \cite{Lemarechal1997} of the Fenchel conjugate of $\phi$, namely
$$
\phi^*_{S_{\theta}}(u) \defin \sup_{\lambda \in \Omega} \Big[   \lambda u- S_\theta  \phi(\lambda) - \delta_\mu \frac{\lambda^2}{2} \Big].
$$
The goal is then to find a feasible penalty $\phi$ of which the inverse mapping $u \mapsto [\phi^*_{S_{\theta}}]^{-1}(u)$ would minimize the regret bound. 
For instance, the inverse mapping would scale as $u \mapsto \sqrt{u}$ when using the squared $\ell_2$ norm over $\Omega = \Real^+$. We defer to future work the study of this admissible family of penalties.

\paragraph{Which enclosing set for $\Xcal$?}
Our current analysis relies on the idea that instead of having to perform a projection on $\Xcal$ in each update (which could be computationally costly and perhaps intractable in some cases), we restrict the iterates $\xb_t$ to remain within a simpler convex set $\Bcal \supseteq \Xcal$. While we assumed de facto an Euclidean ball for $\Bcal$, we could consider sets enclosing $\Xcal$ more tightly, while preserving the appealing computational properties. Having a principled methodology to choose $\Bcal$ and carefully assessing its impact on the regret bounds is an interesting avenue for future research.  

\bibliographystyle{plain}
\bibliography{MainBibliography}

\end{document}